\documentclass[11pt]{article}

\usepackage[final]{acl}

\usepackage{times}
\usepackage{latexsym}

\usepackage[T1]{fontenc}

\usepackage[utf8]{inputenc}

\usepackage{microtype}

\usepackage{inconsolata}

\usepackage{graphicx}



\usepackage{hyperref}
\usepackage{url}

\usepackage{graphicx}
\usepackage{subfigure}
\usepackage{booktabs} 
\usepackage{multirow}
\usepackage{bm}
\usepackage{xcolor}
\usepackage{amsmath}
\usepackage{amsthm}
\usepackage{subcaption}

\usepackage{caption}
\usepackage{float} 
\usepackage{subcaption}
\usepackage{amsmath,amssymb,amsfonts}
\usepackage{algorithmic}
\usepackage{textcomp}
\usepackage{bm}
\usepackage{amsfonts}
\usepackage[table]{xcolor}
\usepackage{diagbox}
\usepackage{booktabs}
\usepackage[normalem]{ulem}
\usepackage{multirow}
\usepackage{pifont}
\usepackage{enumitem}
\usepackage{booktabs}
\usepackage{amssymb}
\usepackage{subcaption}
\usepackage{wrapfig}
\usepackage{makecell}

\usepackage{algorithm}
\usepackage{algorithmic}
\usepackage{wrapfig}
\usepackage{stackengine}
\usepackage{comment}
\usepackage{mathtools}
\newtheorem{lemma}{Lemma}

\newtheorem{proposition}{Proposition}
\newtheorem{remark}{Remark}

\usepackage{amsmath}
\usepackage{graphicx}
\usepackage{multirow}
\newcommand{\cmark}{\textcolor{green!50!black}{\ding{51}}}
\newcommand{\xmark}{\textcolor{red}{\ding{55}}}

\usepackage{tikz}
\usepackage{subcaption}
\usepackage{chngcntr}
\usepackage{multicol}

\title{End-to-End On-Device Quantization-Aware Training for LLMs at Inference Cost}

\author{Qitao Tan$^{1}$ ~~ Xiaoying Song$^{2}$ ~~ Ningxi Cheng$^{1}$ ~~ Guoming Li$^{1}$ ~~ Jun Liu$^{3}$ \\ \textbf{Lingzi Hong}$^{2}$ ~~ \textbf{Caiwen Ding}$^{4}$ ~~ \textbf{Jundong Li}$^{5}$ ~~ \textbf{Xiaoming Zhai}$^{1}$ ~~ \textbf{Shaoyi Huang}$^{6}$ \\ \textbf{Jin Lu}$^{1}$ ~~ \textbf{Ninghao Liu}$^{1}$ ~~ \textbf{Wei Niu}$^{1}$ ~~ \textbf{Geng Yuan}$^{1}$\\
$^{1}$University of Georgia ~~ $^{2}$University of North Texas ~~  $^{3}$Northeastern University \\
$^{4}$University of Minnesota ~~ $^{5}$University of Virginia ~~  $^{6}$Stevens Institute of Technology \\
}

\begin{document}
\maketitle
\begin{abstract}
Quantization reduces the deployment cost of large language models (LLMs). However, post-training quantization (PTQ) cannot fine-tune model parameters and often degrades accuracy in low-bit settings, while quantization-aware training (QAT) requires backpropagation and incurs high memory overhead. We propose ZeroQAT, a zeroth-order QAT framework supporting both weight and activation quantization. By using forward-only gradient estimation, ZeroQAT removes backpropagation and substantially reduces memory and computation while retaining end-to-end optimization benefits. We further introduce a lightweight variant of ZeroQAT for quantized fine-tuning, which freezes and pre-quantizes most parameters to further cut memory usage. Experiments show that ZeroQAT consistently outperforms representative PTQ and QAT baselines while requiring significantly less memory. For example, ZeroQAT enables fine-tuning of a 13B model at extremely low bit-widths (e.g., 2-4 bits) on a single 8GB GPU, and even allows fine-tuning a 6.7B model on a OnePlus 12 smartphone, demonstrating its practicality for end-to-end QAT on resource-limited edge devices.
\end{abstract}

\section{Introduction}
\label{intro}


Large language models (LLMs) have emerged as essential tools for advancing natural language understanding and generation, driving progress in both research and industrial applications~\citep{yang2019end, liu2019roberta, talmor2018commonsenseqa, chowdhery2023palm, zheng2020end}. Despite their transformative potential, training and deploying these models incur extremely high computational and memory costs. Such requirements not only constrain accessibility and scalability but also limit practicality in resource-constrained environments, including mobile and edge devices, embedded systems, and even enterprise servers with strict hardware or budget limitations~\citep{zeng2024flightllm, chen2024understanding, tan2025perturbation}.

To address these challenges, model compression has been widely studied, with quantization emerging as one of the most effective techniques for deployment. Existing methods are mainly divided into post-training quantization (PTQ) and quantization-aware training (QAT). PTQ is simple and practical since it avoids retraining, whereas QAT typically achieves higher accuracy when sufficient resources are available. However, for LLMs, the memory cost of QAT is often prohibitive~\citep{team2025gemma}, making PTQ the dominant choice in practice due to feasibility rather than superiority.

In low-bit settings, the key challenge is adapting to the distribution shifts introduced by quantization. Range-based PTQ methods~\citep{jacob2018quantization,nagel2019data,xiao2023smoothquant} provide limited adaptation and often suffer from noticeable accuracy loss. More advanced approximation-based PTQ methods~\citep{nagel2020up, li2021brecq, frantar2022gptq, shao2023omniquant} better match full-precision outputs, but they still rely on indirect block-wise reconstruction objectives rather than true end-to-end optimization. This often leads to cumulative errors and objective mismatch, especially in low-bit regimes. The problem is even more severe for fine-tuned models, which are highly task-specific and sensitive to quantization perturbations~\citep{dong2021should}.

QAT offers a more principled solution by incorporating quantization effects during training, enabling the model to adapt to quantization errors. Yet its reliance on backpropagation makes it difficult to apply to large models. Recent zeroth-order (ZO) optimization methods estimate gradients using only forward passes, avoiding the storage of activations and optimizer states and thus greatly reducing memory usage. This raises a natural question: \emph{Can ZO enable memory-efficient QAT for high-quality low-bit quantization of LLMs?}

\begin{table*}[t]
\small
\centering
\caption{Comparison of our method with existing methods. PEFT indicates parameter-efficient fine-tuning. WO and WA indicate weight-only and weight-activation quantization, respectively.
}
\begin{tabular}{lccccc}
\toprule
\multirow{2}{*}{\textbf{Method}} & \multirow{2}{*}{\textbf{Category}} & \multirow{2}{*}{\textbf{\begin{tabular}[c]{@{}c@{}}Quant\\ Support\end{tabular}}} & \multicolumn{2}{c}{\textbf{Low-bit performance}}  & \multirow{2}{*}{\textbf{\begin{tabular}[c]{@{}c@{}}Memory\\ Efficiency\end{tabular}}} \\
                                 &                                    &                                                                                   & \textbf{Pre-train}    & \textbf{Fine-tune}    &                                                                                       \\ \midrule
SmoothQuant                      & Range PTQ                          & WA                                                                 & \xmark & \xmark & High                                                                                  \\
GPTQ                             & Approx PTQ                         & WO                                                                            & \xmark & \xmark & High                                                                                  \\
OmniQuant                        & Approx PTQ                         & WA                                                                 & \cmark & \xmark & Moderate                                                                              \\ \midrule
LLM-QAT                          & Full QAT                           & WA                                                                 & \cmark & \cmark & Low                                                                                   \\
QLoRA                            & PEFT QAT                           & WO                                                                            & \cmark & \cmark & Moderate                                                                              \\
EfficientQAT                     & PEFT QAT                           & WO                                                                            & \cmark & \cmark & High                                                                                  \\
ZeroQAT                          & Full/PEFT QAT                      & WA                                                                 & \cmark & \cmark & High                                                                                  \\ \bottomrule
\end{tabular}
\label{1_compare}
\end{table*}

In this work, we propose ZeroQAT, the first end-to-end QAT framework that supports on-device quantization of both low-bit weights and activations. Here, end-to-end carries two meanings. First, the entire training and inference pipeline can be executed directly on-device, without offloading any computation to external servers. This enables users to adapt models locally under strict resource constraints. Second, end-to-end means that all trainable parameters are jointly optimized using the downstream task loss, rather than being updated block by block with auxiliary reconstruction objectives.

As shown in Table~\ref{1_compare}, ZeroQAT reduces the resource burden of conventional QAT while mitigating the accuracy loss commonly seen in PTQ. Unlike prior methods that require massive computing resources~\citep{liu2023llm}, ZeroQAT updates model parameters using gradients estimated purely from forward passes, reducing memory usage to inference-level and making QAT feasible even on edge devices. It further integrates learnable weight clipping and activation transformations, optimized jointly with model parameters via ZO. Moreover, a lightweight variant is devised for further memory reduction. Experiments on both quantized pre-training and fine-tuning show that ZeroQAT consistently outperforms representative PTQ and QAT baselines. For instance, it improves accuracy by 5.1\% on average over five zero-shot tasks and 9.1\% on four downstream tasks under 2-bit weight-only quantization. More importantly, ZeroQAT overcomes the memory barrier of QAT, enabling training of a 13B LLM on a single 8GB low-end GPU and even fine-tuning a 6.7B model on a OnePlus 12 smartphone. This capability makes end-to-end on-device QAT practical on resource-constrained edge devices.

In summary, our major contributions are as follows: 
1) We conduct a preliminary study of PTQ and QAT in low-bit pre-training and fine-tuning, revealing their weaknesses and causes of performance degradation. 2) We propose ZeroQAT, a novel end-to-end zeroth-order QAT framework that achieves high-quality low-bit quantization with inference-level cost. 3) We conduct extensive evaluation across LLM architectures, datasets, and quantization settings, showing consistent accuracy and memory improvements over prior PTQ and QAT baselines.
4) We validate on mobile devices, where ZeroQAT can fine-tune OPT-6.7B on the OnePlus 12 smartphone while full-precision ZO fine-tuning is infeasible, demonstrating its practicality for real-world deployment.

\begin{table*}[th]
\centering
\small
\setlength{\tabcolsep}{4.2pt}
\caption{Results of applying different quantization methods on LLama-2-7B. $^{\ddagger}$ indicates that the method is intrinsically not suitable for the setting; we report these results to illustrate its limitations.  }
\begin{tabular}{lccccccc}
\toprule
\multirow{2}{*}{Method} & \multirow{2}{*}{Category} & \multicolumn{3}{c}{Quantized Pre-training (PPL $\downarrow$)} & \multicolumn{3}{c}{Quantized Fine-tuning (Acc $\uparrow$)} \\
                        &                           & W6A6              & W2A16               & W4A4~              & W6A6              & W2A16g128              & W4A4             \\ \midrule
ZO (FP16)                    & -                         & \multicolumn{3}{c}{5.47~~~~~~}                                    & \multicolumn{3}{c}{66.0~~}                                  \\ 
Zero-shot                    & -                         & \multicolumn{3}{c}{-~~~~~~}                                    & \multicolumn{3}{c}{41.3~~}                                  \\ \midrule
SmoothQuant             & Range-based PTQ                 & 6.20              & 100.23$^{\ddagger}$              & 83.12~~             & 57.2             & 27.7$^{\ddagger}$               & 32.9$^{\ddagger}$             \\
OmniQuant               & Approx-based PTQ                 & 5.87              & 37.37               & 14.26~~             & 63.9              & 40.6$^{\ddagger}$               & 38.8$^{\ddagger}$             \\
EfficientQAT            & QAT                       & 5.60              & 33.40               & 76.32$^{\ddagger}$             & 66.4              & 45.4~~               & 28.6$^{\ddagger}$             \\ \midrule
ZeroQAT                 & QAT                       & 5.76              & 29.61               & 12.95~~             & 65.3              & 54.1~~               & 55.7             \\ \bottomrule
\end{tabular}
\label{table_com}
\vspace{-5pt}
\end{table*}
\section{Background and Related Works}

\textbf{Quantization.} In this work, we mainly study the widely used uniform quantization~\citep{jacob2018quantization} for its better efficiency. The quantization process can be formulated as:
\begin{equation}
\small
\overline{\mathbf{X}}_{\text {INT}}=\text{clamp}(\left\lceil\frac{\mathbf{X}_{\mathrm{FP} 16}}{\Delta}\right\rfloor+z, Q_{N} , Q_{P})
\end{equation}
where $\mathbf{X}$ is the floating-point tensor, $\overline{\mathbf{X}}$ is the quantized counterpart, $\lceil \cdot \rfloor$ is rounding operation, $N$ is the target bit number, $\Delta$ and $z$ denote the step size and zero-point offset value respectively. For symmetric quantization, $Q_N=-2^{N-1}$, $Q_P=2^{N-1}-1$, $\Delta=\frac{\max(|\mathbf{X}|)}{Q_P}$ and $z=0$. Whereas for asymmetric quantization, $Q_N=0$, $Q_P=2^{N}-1$, $\Delta=\frac{\max(|\mathbf{X}|) - \min(|\mathbf{X}|)}{Q_P}$ and $z=-\lceil\frac{\min(|\mathbf{X}|)}{\Delta}\rfloor$. In this paper, we focus on the asymmetric quantization scheme for its better accuracy.

\textbf{Block-wise calibration.} Block-wise calibration strategy is the most widely adopted approach in approximation-based PTQ, because it is relatively efficient in terms of memory, computation, and data usage. The key idea is to minimize quantization error via reconstruction objectives. For example, the widely used block-wise reconstruction loss minimizes the squared error, relative to the full precision layer output~\citep{li2021brecq, shao2023omniquant}. Formally, when both weights and activations are quantized, this can be stated as
\begin{equation}
\small
    \arg \min_{\overline{W}^{l}}\|W^{l}X^{l}-\overline{W}^{l}\overline{X}^{l}\|_{2}^{2}.
\label{eq1}
\end{equation}
where $\overline{W},\overline{X}$ are the quantized version of weight and activations, $l$ indicates the $l$-th layer.

We present our related works section in Appendix~\ref{related_work}.

\section{Challenge of Existing Methods}
\label{analysis}

\subsection{Challenge of Existing PTQ Methods}

\textbf{Range-based PTQ.} These methods rescale or clip the ranges of weight and activation to reduce quantization error. They are computationally efficient and perform reasonably well at moderate bit-width. For example, SmoothQuant~\citep{xiao2023smoothquant} achieves a perplexity of 6.20 in W6A6 (i.e., quantization using 6-bit weight and 6-bit activation), close to the full-precision 5.47 (Table~\ref{table_com}). However, their limited adaptation to distributional and semantic characteristics leads to severe degradation at low bit-widths. For example, under W4A4, SmoothQuant’s perplexity deteriorates to 83.12 versus 5.47 in full precision.

\begin{figure*}[t]
    \centering
        \begin{minipage}{0.33\linewidth}
		\centering
		\includegraphics[width=1\linewidth]{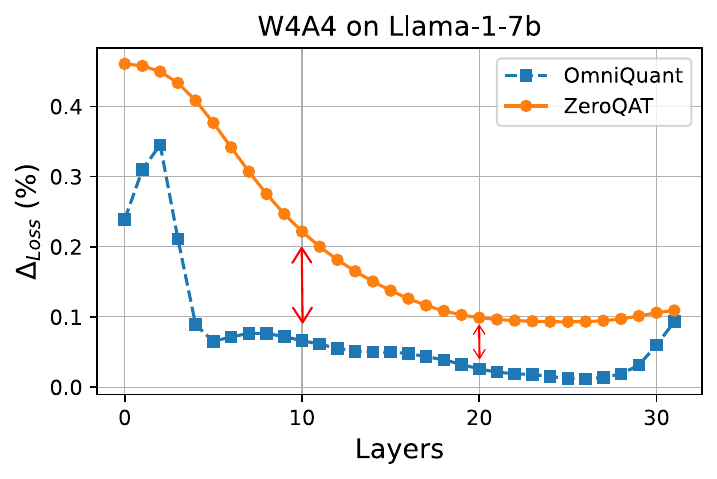}
	\end{minipage}
        \begin{minipage}{0.33\linewidth}
		\centering
		\includegraphics[width=1\linewidth]{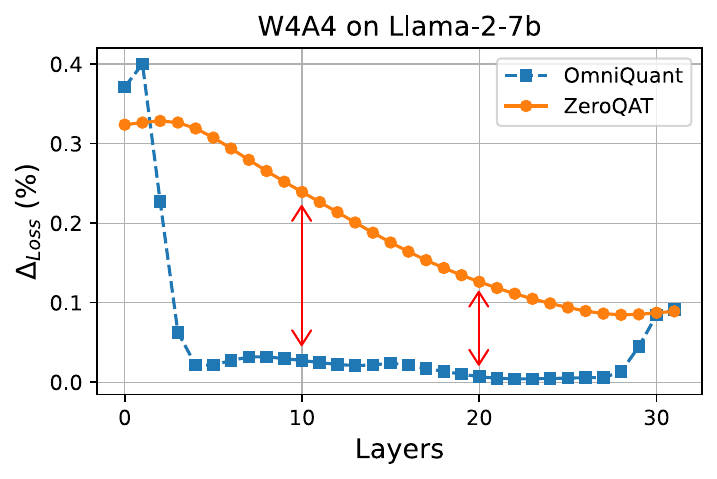}
	\end{minipage}
	\begin{minipage}{0.32\linewidth}
		\centering
		\includegraphics[width=1\linewidth]{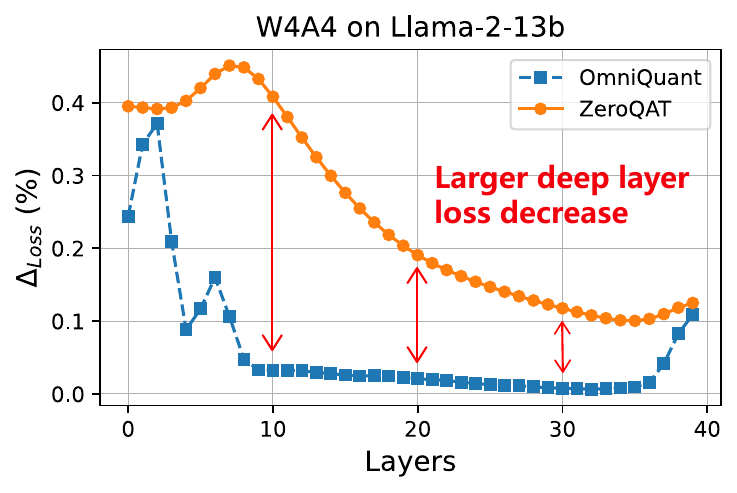}
	\end{minipage}
    
    \caption{Comparison of the block-wise reconstruction loss reduction between OmniQuant~\citep{shao2023omniquant} and our method. The X-axis is the index of layer, Y-axis measures the ratio of loss decrease.
    }
    \label{3_error_prop}
\vspace{-5pt}
\end{figure*}

\textbf{Approximation-based PTQ.} 
These methods narrow the gap between quantized and full-precision outputs via techniques such as learned rounding or reconstruction, adapting to data distributions and model behavior. However, two issues remain and are exacerbated in low-bit quantization settings. 


\begin{figure}
\centering
    \includegraphics[width=0.47\textwidth]{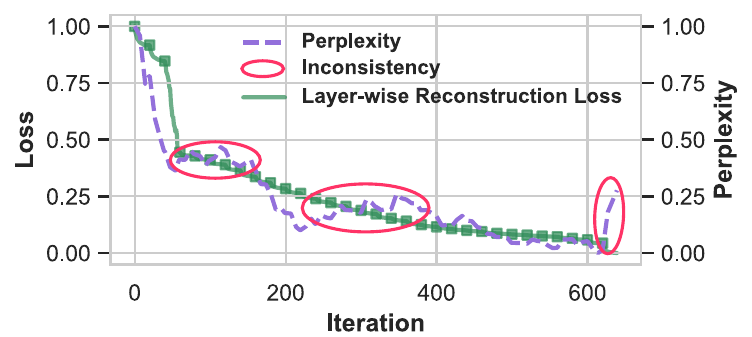}
    \caption{Objective inconsistency between the reconstruction loss used by approximation-based PTQ and final evaluation metrics.}
\label{3_inconsistency}
\vspace{-10pt}
\end{figure}
Here, we take a representative approximation-based PTQ method, OmniQuant~\citep{shao2023omniquant}, as an example to illustrate the two issues. 1) \underline{Cumulative error propagation.} To measure error propagation, we compute relative loss reduction across layers, $\Delta_{Loss} = (\mathcal{L}_{before} - \mathcal{L}_{after}) / \mathcal{L}_{before}$, where $\mathcal{L}_{before}$ and $\mathcal{L}_{after}$ denote reconstruction loss before and after optimization. As shown in Figure~\ref{3_error_prop}, OmniQuant improves shallow layers, but benefits diminish in deeper ones, since each layer is optimized on activations already perturbed by prior quantization noise, making it increasingly difficult to suppress the reconstruction error. This cumulative error propagation constrains the overall quantization quality.
2) \underline{Objective inconsistency.} OmniQuant uses block-wise reconstruction loss (see Eq.\ref{eq1}) as training objective, assuming lower reconstruction loss is aligned with lower perplexity and better downstream accuracy. However, as shown in Figure~\ref{3_inconsistency}, this alignment does not always hold, in several training stages (highlighted in red), reconstruction loss decreases while perplexity fluctuates. This indicates that local layer-level improvements do not reliably translate into global task-level gains, making reconstruction loss a suboptimal proxy for end-to-end performance, especially under low-bit quantization.

\textbf{Failure on fine-tuned model.} 
When PTQ is applied to fine-tuned LLMs, it often fails to preserve task accuracy under low-bit settings. As shown in Table~\ref{table_com}, SmoothQuant maintains moderate accuracy at W6A6 (57.2\% vs. 66.0\% in FP16) but drops to 32.9\% at W4A4. Similarly, OmniQuant achieves 63.9\% at W6A6, close to FP16, yet falls to 38.8\% at W4A4 despite optimization-based techniques. These results indicate that while PTQ remains viable at moderate bit-widths, its effectiveness collapses under aggressive compression, in some cases nearly destroying task performance.

\subsection{Challenge of Existing QAT Methods}

Compared with PTQ, QAT offers stronger adaptation by compensating for quantization errors during training. However, its computational and memory costs are prohibitive for LLMs~\citep{liu2023llm}. To reduce this overhead, later works combine QAT with parameter-efficient methods such as LoRA~\citep{dettmers2023qlora, xu2023qa, li2023loftq} or update only quantizer parameters~\citep{chen2024efficientqat}, achieving competitive results in weight-only quantization. Yet their effectiveness drops in low-bit joint weight-activation settings, as shown in Table~\ref{table_com}. EfficientQAT maintains reasonable perplexity at W6A6 (5.60) and W2A16 (33.40), but degrades sharply at W4A4 (76.32), highlighting the difficulty of modeling dynamic activations.

Overall, although QAT methods can surpass PTQ in some settings, they have not consistently delivered strong results for both weight and activation quantization at aggressive bit-widths under realistic resource constraints. Motivated by this gap, we develop a ZO-based QAT framework that, to the best of our knowledge, is the first to maintain superior accuracy in both low-bit weight and activation settings.

\section{ZeroQAT}

\subsection{Quantization-aware Zeroth-order Optimization}

Unlike conventional first-order optimization that computes gradients via backpropagation, zeroth-order (ZO) optimization estimates them using only function queries through finite differences~\citep{chen2023deepzero, liu2018zeroth, ye2018hessian}. This avoids storing activations, backward gradients, and optimizer states, greatly reducing memory costs in LLM fine-tuning. For each random direction, ZO requires only two forward passes to approximate the gradient: 
\begin{equation}
\small
\label{equ:1}
    \hat{\nabla} \mathcal{L}(\overline{W})=\frac{1}{q} \sum_{i=1}^q\left[\frac{\mathcal{L}\left(Q(W+\epsilon u_{i})\right)-\mathcal{L}\left(Q(W-\epsilon u_{i})\right)}{2 \epsilon} u_{i}\right]
\end{equation}
where $Q$ is the quantizer, $\overline{W}$ is the quantized parameters, $u_{i}\in \mathcal{N}(0, \mathbf{I})$ is a random perturbation, $q$ is the number of directions, and $\epsilon > 0$ is a small scalar.

Following QAT practice, we maintain full-precision weights while using their quantized counterparts in forward passes. Unlike FO-QAT, ZeroQAT does not require the straight-through estimator (STE)~\citep{bengio2013estimating}, since gradients are estimated directly via zeroth-order finite differences, bypassing the non-differentiability of the quantizer. Given a learning rate $\eta$ and a mini-batch $\mathcal{B}_{t}$ at iteration $t$, the update rule becomes:
\begin{equation}
\small
\label{equ:2}
    W_{t+1} = W_{t} - \eta \hat{\nabla} \mathcal{L}(\overline{W}_{t};\mathcal{B}_{t}).
\end{equation}

In ZeroQAT, the ZO estimator remains unbiased with respect to the gradient of a smoothed quantized objective, which ensures standard convergence guarantees. In contrast, QAT methods based on the STE rely on a hand-crafted surrogate gradient that introduces inherent bias. This bias becomes particularly severe in low-bit regimes, where the true smoothed gradients are already small but STE still produces large surrogate updates, leading to unstable or suboptimal convergence. A formal analysis and quantitative bounds on this bias are provided in Appendix~\ref{theory}.


\subsection{Adaptive Outlier Mitigation}

\textbf{Adaptive per-channel scaling.} Activation outliers enlarge the dynamic range of specific channels and degrade quantization accuracy. Prior work~\citep{xiao2023smoothquant, wei2022outlier, shao2023omniquant, sunflatquant} mitigates this issue by applying channel-wise scaling $\operatorname{diag}(\mathbf{c}) \in \mathbb{R}^{n}$ to activations and weights in a mathematically equivalent form, 
where $\mathbf{c}$ is a statistical or learnable scaler, thereby transferring activation quantization difficulty to weights with more uniform distributions.

\textbf{Adaptive weight \& activation clipping.} To further suppress residual outliers after per-channel smoothing, we introduce learnable clipping thresholds for both weights and activations. For a generic tensor $\mathbf{T}\in{\mathbf{X},\mathbf{W}}$ in a linear layer, we apply asymmetric learnable clipping in a unified form,
$\bar{\mathbf{T}}=\operatorname{clamp}(\mathbf{T},\alpha_{\mathbf{T}}\cdot Q,\beta_{\mathbf{T}}\cdot Q)$,
where $\alpha_{\mathbf{T}}, \beta_{\mathbf{T}}\in(0,1)$ is a layer-specific clipping threshold and jointly optimized with other parameters. This formulation enables end-to-end adaptation of clipping ranges for both weights and activations, yielding lower quantization error without additional inference overhead. Intuitively, for weight or activation distributions with near-uniform distributions after smoothing, $\alpha$ and $\beta$ converge to similar values, resulting in a tight clipping range that preserves precision. In contrast, for biased distributions, $\alpha$ and $\beta$ adapt to asymmetrically clip the dynamic range, thereby mitigating the impact of outliers.

\begin{table*}[htbp]
\centering
\small
\setlength{\tabcolsep}{4pt}
\renewcommand{\arraystretch}{1}
\caption{Weight-only and weight-activation quantization results of LLama-series models on two datasets: WikiText2 (WIKI), and C4. 
}
\begin{tabular}{llcccccccccc}
\toprule
\multicolumn{2}{l}{\textbf{LLama / PPL $\downarrow$}}   & \multicolumn{2}{c}{\textbf{LLama1-7B}}                & \multicolumn{2}{c}{\textbf{LLama1-13B}}               & \multicolumn{2}{c}{\textbf{LLama2-7B}}                & \multicolumn{2}{c}{\textbf{LLama2-13B}}               & \multicolumn{2}{c}{\textbf{LLama3-8B}}                \\
\multicolumn{2}{l}{Task}                                & WIKI                      & C4                        & WIKI                      & C4                        & WIKI                      & C4                        & WIKI                      & C4                        & WIKI                      & C4                        \\ \midrule
FP16                                      & -           & 5.68                      & 7.08                      & 5.09                      & 6.61                      & 5.47                      & 6.97                      & 4.88                      & 6.46                      & 6.14                      & 9.45                      \\ \midrule
\multirow{4}{*}{W2A16}                    & RTN         & \multicolumn{1}{c}{1.1e5} & \multicolumn{1}{c}{1.3e5} & \multicolumn{1}{c}{6.8e4} & \multicolumn{1}{c}{5.6e4} & \multicolumn{1}{c}{3.8e4} & \multicolumn{1}{c}{4.8e4} & \multicolumn{1}{c}{4.2e4} & \multicolumn{1}{c}{1.2e3} & \multicolumn{1}{c}{4.2e4} & \multicolumn{1}{c}{5.6e4} \\
                                          & GPTQ        & 5.6e4                     & 689.13                    & 5.5e3                     & 6.97                      & 7.7e3                     & NAN                       & 2.1e3                     & 323.12                    & 4.2e3                     & 1.2e3                    \\
                                          & OmniQuant   & 15.47                     & 24.89                     & 13.21                     & 18.31                     & 37.37                     & 90.64                     & 17.21                     & 26.76                     & 28.48                     & 56.64                     \\
\rowcolor[gray]{.92}                      & ZeroQAT     & \textbf{9.65}            & \textbf{13.55}            & \textbf{9.09}            & \textbf{11.87}            & \textbf{17.18}            & \textbf{28.36}            & \textbf{8.67}            & \textbf{13.68}            & \textbf{9.14}            & \textbf{12.89}            \\ \midrule
\multicolumn{1}{c}{\multirow{5}{*}{W4A4}} & SmoothQuant & 25.25                     & 32.32                     & 40.05                     & 47.18                     & 83.12                     & 77.27                     & 35.88                     & 43.19                     & 210.19                     & 187.93                     \\
\multicolumn{1}{c}{}                      & OmniQuant   & 11.26                     & \textbf{14.51}            & 10.87                     & 13.78                     & 14.26                     & 18.02                     & 12.30                     & 14.55                     & 13.41                     & 16.83                     \\
\multicolumn{1}{c}{}                      & Spinquant  & 6.44                     & 8.13            & 5.91                     & 7.79                     & 6.14                     & 9.19                     & 5.44                     & 7.48                     & 7.96                     & 13.45                     \\
\multicolumn{1}{c}{}                      & FlatQuant   & 6.33                     & \textbf{7.79}            & \textbf{5.47}                     & \textbf{7.37}                     & 5.79                     & \textbf{7.79}                     & \textbf{5.12}                     & \textbf{7.11}                     & 6.98                     & 11.13                     \\
\rowcolor[gray]{.92}\multicolumn{1}{c}{} & ZeroQAT     & \textbf{6.17}            & 8.08                     & 5.66            & 7.56            & \textbf{5.78}            & 7.93            & 5.27            & 7.43            & \textbf{6.12}            & \textbf{10.43}            \\ \bottomrule
\end{tabular}
\label{tab:llama_quant_results}
\end{table*}

\begin{table*}[ht]
\centering
\small
\setlength{\tabcolsep}{4pt}
\renewcommand{\arraystretch}{1}
\caption{Weight-only and weight-activation quantization results of LLama models. This table reports the accuracy of 5 zero-shot tasks. Refer to Table~\ref{tab:llama_quant_acc_main} for more results.
}
\begin{tabular}{llcccccc}
\toprule
\textbf{LLama / Acc $\uparrow$} & \textbf{Method} & \textbf{PIQA} & \textbf{ARC-e} & \textbf{ARC-c} & \textbf{HellaSwag} & \textbf{Winogrande} & \textbf{Avg.}  \\ \midrule
\multirow{4}{*}{LLama-1-7B}     & FP16            & 77.47         & 72.38          & 41.46          & 73.00              & 67.07               & 65.26          \\
                                & Spinquant      & 73.81         & 66.58          & 34.16          & 67.02              & 59.95               & 60.30          \\
                                & FlatQuant       & 75.73         & 68.87          & 36.88          & 69.31              & 65.46               & 63.17          \\
\rowcolor[gray]{.92}            & ZeroQAT         & 76.98         & 69.32          & 36.19          & 70.86              & 64.87               & \textbf{63.64} \\ \midrule
\multirow{4}{*}{LLama-2-7B}    & FP16            & 79.05         & 74.54          & 46.16          & 75.98              & 69.06               & 68.96          \\
                                & Spinquant      & 75.30         & 65.32          & 39.42          & 71.45              & 63.46               & 62.99          \\
                                & FlatQuant       & 77.26         & 72.05          & 43.26          & 73.64              & 69.53               & \textbf{67.14}          \\
\rowcolor[gray]{.92}            & ZeroQAT         & 76.86         & 68.67          & 42.39          & 74.43              & 68.35               & 66.14 \\ \midrule
\multirow{4}{*}{LLama-3-8B}    & FP16            & 80.74         & 77.57          & 53.50          & 79.12              & 72.93               & 72.77          \\
                                & Spinquant      & 76.66         & 71.38          & 45.73          & 74.07              & 66.38               & 66.84          \\
                                & FlatQuant       & 79.16         & 75.80          & 50.00          & 76.80              & 72.69               & 69.89          \\
\rowcolor[gray]{.92}            & ZeroQAT         & 79.48         & 77.24          & 51.08          & 77.06              & 71.55               & \textbf{71.28} \\ \bottomrule
\end{tabular}
\vspace{-5pt}
\label{tab:llama_quant_acc}
\end{table*}

\textbf{Adaptive transformation.} Following prior work~\citep{ashkboos2024quarot, liuspinquant}, we apply linear transformations to smooth activations and suppress outliers before quantization, thereby improving the flatness of both weight and activation distributions. Concretely, given a linear layer $\mathbf{Y}=\mathbf{XW}$, the transformation can be written in a mathematically equivalent form as
$\mathbf{Y}=(\mathbf{X}\mathbf{P}^{-1})(\mathbf{P}\mathbf{W})$,
where the transformation matrix $\mathbf{P}$ redistributes activation magnitudes and mitigates outliers. 
For efficient implementation, we are inspired by FlatQuant~\citep{sunflatquant} and factorize large affine transformations via matrix decomposition, which significantly reduces memory consumption and computational overhead.

In practice, we design six transformations $\mathcal{P}=\{P_a, P_o, P_k, P_v, P_f, P_d\}$, which are applied to the QKV projections, output projection, key cache, value cache, feed-forward layer, and down-projection layer, respectively. To minimize runtime overhead, different transformations are selectively fused or merged into the corresponding weights, ensuring that the overall inference speed is not noticeably affected.

For all the quantization parameters, previous works generally use hand-crafted or block-wise calibrated scaling, which lacks end-to-end optimization and often leads to suboptimal performance. In contrast, our QAT framework jointly optimizes these parameters along with model weights, enabling adaptive suppression of outliers during training. 


\begin{figure}
\centering
    \includegraphics[width=0.48\textwidth]{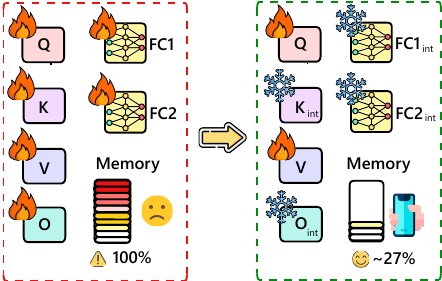}
    \caption{lightweight ZeroQAT for fine-tuning.}
\label{3_quant_illu}
\vspace{-10pt}
\end{figure}

\subsection{Lightweight ZeroQAT for Memory Reduction in Quantized Fine-tuning}

We further propose a lightweight variant of ZeroQAT designed specifically for quantized fine-tuning to substantially reduce the fine-tuning memory footprint. Unlike backpropagation-based methods, where memory is dominated by weights, activations, and optimizer states, ZeroQAT’s cost mainly comes from the parameters actively updated during fine-tuning. Pre-quantizing the entire model could further reduce memory, but this fails in practice. As small ZO perturbations are rounded away while large ones destabilize training, making naive full-model pre-quantization unsuitable.

To overcome this, we introduce a lightweight variant. Most parameters are frozen and pre-quantized, while only the query (Q) and value (V) matrices of attention blocks are kept in full precision, as illustrated in Figure~\ref{3_quant_illu}. Thus, memory use comes from the full-precision Q and V plus quantized frozen weights. This design substantially reduces the fine-tuning footprint while retaining sufficient trainable capacity for adaptation. Although pre-quantizing frozen weights brings only limited memory savings in FO due to the dominant memory cost coming from activation and optimizer, it becomes particularly beneficial in the ZO setting. As model weights account for a much larger fraction of the overall memory footprint, pre-quantizing frozen weights becomes particularly beneficial in the ZO setting. Therefore, it enables fine-tuning large models such as OPT-13B under low-bit settings with memory as low as 6.8 GB (in Table~\ref{Memory_time}), far lighter than existing QAT baselines.

\section{Experiment}


We evaluate ZeroQAT on quantized pre-training and fine-tuning, with ablations and efficiency analysis of memory and speed. Experiments are conducted on an NVIDIA A100 GPU and a OnePlus 12 device (Snapdragon 8 Gen 3). Hyperparameter settings are provided in Appendix~\ref{hyper_setting}, and all results are averaged over three runs.

\subsection{ZeroQAT for Quantized Pre-training}
\label{Intrinsic}

\textbf{Training and evaluation.} For the parameters of smoothing (channel-wise scaling and transformation) and clipping, we leverage reconstruction loss for a lightweight initialization, and then jointly train with the model via ZO. For LLama-series weight-only quantization, we retain only clipping. Pre-training uses mixed segments from WikiText2 and C4, with perplexity measured on the pre-training context datasets. We further evaluate zero-shot accuracy on five datasets under GPTQ~\citep{frantar2022gptq} settings with lm-eval-harness~\citep{eval-harness}. More experimental and baseline details are provided in Appendix~\ref{qpt_setting}.

\begin{table*}[ht]
\centering
\small
\setlength{\tabcolsep}{1.0pt}
\caption{Weight-only and weight-activation quantized fine-tuning results of OPT models. The table reports the accuracy or F1 score of the corresponding tasks. g128 indicates the group size is 128.
}
\begin{tabular}{llcccccccccccc}
\toprule
\multicolumn{2}{l}{\textbf{OPT / Acc $\uparrow$}}        & \multicolumn{4}{c}{\textbf{OPT-2.7B}}                                              & \multicolumn{4}{c}{\textbf{OPT-6.7B}}                                              & \multicolumn{4}{c}{\textbf{OPT-13B}}                          \\
\multicolumn{2}{l}{Task}                                 & SST-2         & CB            & SQuAD         & DROP                               & SST-2         & CB            & SQuAD         & DROP                               & SST-2         & CB            & SQuAD         & DROP          \\ \midrule
Zero-shot                                 &              & 56.3          & 50.0          & 29.8          & \multicolumn{1}{c|}{10.0}          & 64.2          & 50.0          & 37.9          & \multicolumn{1}{c|}{13.1}          & 58.8          & 46.4          & 46.2          & 14.6          \\
FP16 (ZO)                                 & -            & 90.0          & 69.6          & 68.7          & \multicolumn{1}{c|}{22.9}          & 90.2          & 71.4          & 76.0          & \multicolumn{1}{c|}{26.4}          & 91.4          & 67.9          & 84.7          & 30.9          \\ \midrule
\multirow{6}{*}{W2A16g128}                & RTN          & 44.4          & 44.6          & 0.0           & \multicolumn{1}{c|}{0.0}           & 59.2          & 50.0          & 0.0           & \multicolumn{1}{c|}{0.0}           & 53.5          & 50.0          & 0.0             & 0.0             \\
& QuZO & 47.2          & 46.3          & 0.0          & \multicolumn{1}{c|}{0.0}          & 60.6          & 55.3          & 0.0          & \multicolumn{1}{c|}{0.0}          & 59.6          & 51.8          & 0.0          & 0.0          \\
                                          & QLoRA        & 61.2          & 51.8          & 0.0           & \multicolumn{1}{c|}{8.2}           & 64.8          & 58.9          & 0.0           & \multicolumn{1}{c|}{0.0}           & 63.8          & 69.6          & 0.0             & 0.0             \\
                                          & OmniQuant    & 72.8          & 55.4          & 16.5          & \multicolumn{1}{c|}{4.4}           & 61.6          & 55.3          & 27.7          & \multicolumn{1}{c|}{12.6}          & 62.6          & 29.8          & 38.8          & 16.4          \\
                                          & EfficientQAT & 76.6          & 57.1          & 29.0          & \multicolumn{1}{c|}{12.6}          & 75.6          & 58.9          & 32.4          & \multicolumn{1}{c|}{14.6}          & 81.2          & 62.5          & 46.7          & 16.9          \\
\rowcolor[gray]{.92}                      & ZeroQAT      & \textbf{87.8} & \textbf{64.8} & \textbf{52.8} & \multicolumn{1}{c|}{\textbf{16.2}} & \textbf{88.6} & \textbf{67.4} & \textbf{61.3} & \multicolumn{1}{c|}{\textbf{22.4}} & \textbf{89.4} & \textbf{63.8} & \textbf{68.8} & \textbf{25.9} \\ \midrule
\multicolumn{1}{c}{\multirow{5}{*}{W4A4}} & QuZO  & 53.4          & 53.6          & 0.0           & \multicolumn{1}{c|}{0.0}           & 60.3          & 57.3          & 0.0          & \multicolumn{1}{c|}{0.0}           & 57.4          & 51.8          & 0.0          & 0.0           \\
& SmoothQuant  & 56.0          & 55.4          & 7.6           & \multicolumn{1}{c|}{5.4}           & 58.8          & 50.0          & 12.8          & \multicolumn{1}{c|}{6.2}           & 57.5          & 52.4          & 13.4          & 7.1           \\
\multicolumn{1}{c}{}                      & OmniQuant    & 59.2          & 60.7          & 22.1          & \multicolumn{1}{c|}{6.7}           & 61.2          & 48.2          & 24.7          & \multicolumn{1}{c|}{11.7}          & 59.2          & 50.0          & 28.8          & 13.5          \\
\multicolumn{1}{c}{}                      & FlatQuant    & 79.7          & 62.5          & 38.7          & \multicolumn{1}{c|}{12.4}                                & 82.3          & 57.1          & 47.6          & \multicolumn{1}{c|}{17.9}                               & 85.7          & 60.7          & 58.9          & 22.8          \\
\rowcolor[gray]{.92}\multicolumn{1}{c}{} & ZeroQAT      & \textbf{88.4} & \textbf{67.8} & \textbf{59.7} & \multicolumn{1}{c|}{\textbf{18.3}} & \textbf{89.3} & \textbf{64.3} & \textbf{64.1} & \multicolumn{1}{c|}{\textbf{21.3}} & \textbf{89.8} & \textbf{66.1} & \textbf{72.4} & \textbf{27.3} \\ \bottomrule
\end{tabular}
\vspace{-5pt}
\label{tab:fine-tuned}
\end{table*}

\textbf{Perplexity results.} We evaluate the intrinsic language modeling performance of quantized models using perplexity, with results for the LLaMA and OPT families reported in Tables~\ref{tab:llama_quant_results} and~\ref{tab:opt_quant_results}. Under the relatively mild W6A6 setting, all methods achieve near lossless performance compared to full precision. In contrast, under more aggressive settings (W2A16 and W4A4), ZeroQAT consistently outperforms baseline methods by enabling stronger adaptation through end-to-end training, resulting in lower perplexity across model families and datasets. These results demonstrate the effectiveness of ZeroQAT in preserving language modeling performance under severe quantization.

\textbf{Zero-shot accuracy results.} Table \ref{tab:llama_quant_acc} reports the zero-shot results of the LLama-series models on five reasoning datasets 
evaluated by accuracy. As expected, the FP16 setting achieves the highest average accuracy, serving as the upper bound. With the 4-bit weight and activation quantization, ZeroQAT outperforms other quantization approaches, yielding higher average accuracy across models. This result demonstrates that ZeroQAT maintains strong task generalization even when quantization is pushed to low-bit precision. 

It is worth noting that although strong baseline methods such as FlatQuant can achieve performance comparable to ZeroQAT in the quantized pre-training setting, the advantages of our method become more evident in the subsequent quantized fine-tuning experiments and efficiency analysis.

\subsection{ZeroQAT for Quantized Fine-tuning}
\label{ft}

\textbf{Training and evaluation.} Following prior work~\citep{pan2024lisa, chen2024efficientqat}, we fine-tune models on a small subset of Alpaca~\citep{taori2023stanford} and evaluate the reasoning performance. 
Moreover, we adopt a few-shot fine-tuning protocol for evaluation following previous work~\citep{malladi2023fine}. 
For evaluation on PTQ methods, we first fine-tune the models in full precision using ZO, and then apply the corresponding quantization method. In contrast, QAT methods, like EfficientQAT and our ZeroQAT, directly produce quantized models during fine-tuning without the need for a separate PTQ stage.
Additional experimental and baseline details are provided in Appendix~\ref{qft_setting}.

\textbf{Results}. The results are summarized in Table~\ref{tab:fine-tuned}. Fine-tuning adapts model parameters to narrow task-specific optima~\citep{dong2021should}, which increases their sensitivity to quantization noise. Consequently, less adaptive PTQ methods suffer from severe degradation in low-bit settings. By comparison, ZeroQAT consistently delivers higher accuracy across all tasks and model scales, in some cases approaching FP16 performance. For example, under the W4A4 setting, ZeroQAT achieves about 89\% accuracy on SST-2 across the three OPT models, whereas the strongest baseline methods achieve less than 80\%. We also conduct fine-tuning on Alpaca, with results shown in Table~\ref{tab:zeroqat}. ZeroQAT again outperforms prior methods across different bit-widths and model sizes. 

\begin{table}[htbp]
\centering
\small
\caption{Averaged accuracy over 5 zero-shot datasets after fine-tuning. Results on MMLU and GSM8K are presented in Table~\ref{mmlu_results} and Table~\ref{gsm8k_results}.}
\begin{tabular}{lccc}
\toprule
\textbf{Method} & \textbf{\#Bits} & \textbf{7B} & \textbf{13B} \\
\midrule
-              & FP               & 67.0        & 69.3         \\ \midrule
QLoRA w/GPTQ    & W2A16         & 31.8        & 32.4         \\
QA-LoRA         & W2A16         & 34.6        & 37.3         \\
IR-QLoRA        & W2A16         & 34.4        & 36.3         \\
PEQA            & W2A16         & 35.2        & 34.8         \\
EfficientQAT    & W2A16         & 49.1        & 52.1         \\
\rowcolor[gray]{.92}ZeroQAT         & W2A16            & \textbf{59.4}        & \textbf{62.7}         \\
\midrule
SmoothQuant       & W4A4             & 37.4        & 41.6         \\
OmniQuant       & W4A4             & 52.3        & 54.2         \\
FlatQuant       & W4A4             & 58.8        & 60.7         \\
\rowcolor[gray]{.92}ZeroQAT         & W4A4             & \textbf{63.1}        & \textbf{66.8}         \\
\bottomrule
\end{tabular}
\label{tab:zeroqat}
\vspace{-10pt}
\end{table}

\begin{table*}[htbp]
\centering
\small
\setlength{\tabcolsep}{5.0pt}
\caption{Memory consumption and wallclock time per update during quantized pre-training and fine-tuning under the W2A16g128 setting. Since ZeroQAT only stores the weights in memory, its memory usage remains unaffected by batch size. Four and sixteen is the ideal batch size in FlatQuant and EfficientQAT, respectively.
}
\begin{tabular}{lcccccccc}
\toprule
\multirow{2}{*}{\textbf{Method}} & \multicolumn{2}{c}{\textbf{OPT-1.3B}} & \multicolumn{2}{c}{\textbf{OPT-2.7B}} & \multicolumn{2}{c}{\textbf{OPT-6.7B}} & \multicolumn{2}{c}{\textbf{OPT-13B}} \\
                                 & Memory            & Time              & Memory            & Time              & Memory             & Time             & Memory             & Time            \\ \midrule                                         
\multicolumn{9}{c}{Quantized Pre-training (avg sequence length = 2048)}                                                                                                                         \\ \midrule
LLM-QAT (bsz=1)                  & 28.8GB            & 1.00s            & 58.6GB            & 1.64s            & $\sim$166GB        & $\sim$5.0s     & $\sim$337GB        & $\sim$15.5s   \\
FlatQuant (bsz=1)                & 7.2GB             & 1.12s            & 8.5GB             & 1.83s            & 13.9GB             & 3.04s           & 20.1GB             & 5.14s          \\
\rowcolor[gray]{.92}ZeroQAT (bsz=1)                  & 3.1GB             & 0.58s            & 6.1GB             & 0.98s            & 14.2GB             & 1.77s           & 26.6GB             & 3.12s          \\
FlatQuant (bsz=4)                & 17.4GB            & 2.55s            & 21.5GB            & 4.03s            & 27.2GB             & 6.35s           & 33.8GB             & 11.81s         \\
\rowcolor[gray]{.92}ZeroQAT (bsz=4)                  & 3.1GB             & 1.72s            & 6.1GB             & 2.76s            & 14.2GB             & 4.48s           & 26.6GB             & 7.74s          \\ \midrule
\multicolumn{9}{c}{Quantized Fine-tuning (max sequence length = 384)}                                                                                                                           \\ \midrule
EfficientQAT (bsz=1)                & 2.1GB             & 0.13s            & 3.1GB             & 0.21s            & 4.4GB             & 0.36s           & 7.3GB             & 0.67s          \\
\rowcolor[gray]{.92}ZeroQAT (bsz=1)                  & 0.8GB             & 0.04s            & 1.5GB             & 0.07s            & 3.7GB             & 0.18s           & 6.8GB             & 0.32s          \\
EfficientQAT (bsz=16)                & 5.9GB            & 0.69s            & 8.1GB            & 1.10s            & 11.9GB             & 1.70s           & 17.2GB             & 3.26s         \\
\rowcolor[gray]{.92}ZeroQAT (bsz=16)                  & 0.8GB             & 0.31s            & 1.5GB             & 0.53s            & 3.7GB             & 0.94s           & 6.8GB             & 1.73s         \\ \bottomrule
\end{tabular}
\label{Memory_time}
\end{table*}
\begin{table*}[htbp]
\centering
\small
\caption{
Evaluation of memory consumption and speed on a OnePlus 12 smartphone under W4A4 quantization. Prompts of 384 tokens are used in inference, and OOM indicates out of memory.
}
\begin{tabular}{llcccccc}
\toprule
\multirow{2}{*}{Stage}       & \multirow{2}{*}{Metrics} & \multicolumn{2}{c}{\textbf{OPT-1.3B}}         & \multicolumn{2}{c}{\textbf{OPT-2.7B}} & \multicolumn{2}{c}{\textbf{OPT-6.7B}} \\
                             &                          & FP16                  & ZeroQAT               & FP16             & ZeroQAT            & FP16             & ZeroQAT            \\ \midrule
\multirow{3}{*}{Fine-tuning} & Latency                  & 11.2s                 & 7.8s                  & 19.6s            & 12.3s              & /                & 29.1s              \\
                             & Weight memory            & 2.6GB                  & 0.9GB                  & 5.4GB             & 1.8GB               & 13.4GB                & 4.6GB               \\
                             & Running memory           & 3.5GB                  & 1.2GB                  & 8.1GB             & 2.6GB               & OOM              & 6.4GB               \\ \midrule
\multirow{2}{*}{Inference}   & Token / s                  & 10.9 & 15.4 & 7.58             & 11.0               & 3.13             & 4.76               \\
                             & Speed up                 & 1.0$\times$                   & 1.41$\times$                   & 1.0$\times$              & 1.45$\times$                & 1.0$\times$              & 1.52$\times$                \\ \bottomrule
\end{tabular}
\label{tab:eva-on-device}
\vspace{-10pt}
\end{table*}

\subsection{Efficiency of ZeroQAT}
\label{efficiency}

To highlight the advantage that our method enables generating a quantized and fine-tuned model through a lightweight end-to-end pipeline, we evaluate the efficiency of ZeroQAT on both a GPU server and a mobile device to demonstrate its practicality across different platforms. Implementing and optimizing such a pipeline on mobile devices is non-trivial due to their limited computational and memory resources. Therefore, we provide additional details on the implementation and optimization of our method on mobile devices in Appendix~\ref{IO}.

\textbf{Server-side efficiency.} 
Table~\ref{Memory_time} compares memory requirements and wallclock time per update between our method and baselines. For quantized pre-training, ZeroQAT reduces memory usage by 89-92\% relative to the costly LLM-QAT, while also accelerating training. Compared to the PTQ method FlatQuant, ZeroQAT offers clear advantages, for instance, it halves memory use (OPT-1.3B: 6.1GB to 3.1GB) and achieves about 1.5× faster updates (OPT-2.7B: 1.49s to 0.98s). For quantized fine-tuning, ZeroQAT’s memory-efficient design requires storing only weights of query and value heads in attention blocks, making memory usage independent of batch size. Against EfficientQAT, it consistently saves memory and improves throughput, especially on smaller models such as OPT-1.3B, reducing memory by 86\% (5.9GB to 0.8GB) and wallclock time by 55\% (0.69s to 0.31s) with the same batch size.

\textbf{On-device efficiency.} 
Table~\ref{tab:eva-on-device} compares FP16 baseline with ZeroQAT under W4A4 for OPT-1.3B, 2.7B, and 6.7B models. The results were collected on a OnePlus 12 smartphone with a Snapdragon 8 Gen 3 SoC and 16GB RAM. ZeroQAT reduces fine-tuning latency by 30\% and 37\% for OPT-1.3B and OPT-2.7B, respectively, while cutting running memory from 3.5GB to 1.2GB and from 8.1GB to 2.6GB. For OPT-6.7B, FP16 fine-tuning is infeasible (OOM), whereas ZeroQAT runs within 6.4GB memory with 29.1s latency. During inference, ZeroQAT further achieves 1.41$\times$-1.52$\times$ higher token throughput, demonstrating its practicality on resource-constrained devices.



\section{Conclusion}

In this paper, we propose ZeroQAT, a zeroth-order QAT framework that supports both weight and activation quantization at extremely low-bit setting. We introduce adaptive smoothing and an adaptive weight quantizer to reduce quantization error, along with a lightweight variant that freezes and quantizes part of the model to lower fine-tuning memory cost. Experiments on quantized pre-training, fine-tuning, and on-device deployment show that ZeroQAT consistently outperforms PTQ and QAT baselines in both accuracy and efficiency, enabling fine-tuning of large LLMs on a OnePlus 12 smartphone under strict memory constraints.

\section{Limitations}
While ZeroQAT demonstrates strong performance across multiple model scales and deployment settings, our current evaluation is limited to LLMs up to 13B parameters, as our method primarily targets resource-limited platforms (e.g., mobile devices). Models larger than this scale (e.g., 30B parameters and beyond) typically fall outside the resource constraints of such platforms and are therefore less aligned with our primary deployment scenario. Nevertheless, extending the proposed framework to larger models remains an interesting direction for future work.

In addition, the requirement for more training iterations is an inherent limitation of our approach, as well as of ZO-based methods in general. Developing more efficient zero-order gradient estimators could substantially reduce the number of required optimization steps. Given that our focus is on memory efficiency rather than speed, and acceleration techniques are orthogonal to and fully compatible with our framework. Exploring these improvements represents an important, but separate, direction for future work.

\section{Claim of LLM Usage}

In this work, large language models (LLMs) were used solely as a general-purpose writing assistant. Their role was limited to correcting grammar, fixing typographical errors, and polishing the language for clarity and readability. 

\section*{Acknowledgment}
This work is partly supported by the National Science Foundation IIS-2223768 and DRL-2507128.

\bibliography{custom}

\clearpage
\appendix
\pagenumbering{alph}

\renewcommand{\thetable}{\thesection.\arabic{table}}
\counterwithin{table}{section}


\section{Related work}
\label{related_work}

\subsection{Model Quantization}

Quantization techniques aim to properly map the original continuous real values to a discrete low-bit format (e.g., INT8 or INT4), leading to significant memory saving and inference acceleration while maintaining the performance~\citep{zhou2016dorefa}. Quantization techniques can be generally divided into two categories: Post-training quantization (PTQ) and quantization-aware training (QAT). The QAT method generally yields better results due to better adaptation capability, but the high retraining cost (in both memory and computation) has discouraged many researchers. Therefore, most of the LLM quantization works focus on PTQ methods, which can be mainly divided into range-based PTQ~\citep{jacob2018quantization,nagel2019data,xiao2023smoothquant} and approximation-based PTQ methods~\citep{nagel2020up,li2021brecq, frantar2022gptq,shao2023omniquant}. The range-based PTQ typically relies on static analysis, where the range (e.g., minimum and maximum values) of weights or activations is collected to determine quantization parameters. The approximation-based PTQ methods, with more adaptation, explicitly frame quantization as an error minimization problem, optimizing quantized parameters to closely approximate the full-precision model outputs.

\subsection{Zeroth-order Optimization}

Zeroth-order optimization (ZO), which estimates gradients using only function evaluations, has emerged as an attractive alternative to classical first-order (FO) methods. Compared to FO approaches, ZO eliminates the need for backpropagation, thereby simplifying implementation and significantly reducing memory consumption. This makes it appealing in scenarios such as adversarial attack and defense~\citep{chen2017zoo, ye2018hessian, verma2023certified}, machine learning explainability~\citep{dhurandhar2018explanations, dhurandhar2019model}, reinforcement learning~\citep{vemula2019contrasting}, and on-chip training~\citep{gu2021efficient}. Despite these successes, ZO optimization has been primarily applied to relatively small-scale problems, since its convergence is generally slower and suffers from high variance due to random search. These challenges are exacerbated in large-scale settings such as LLM fine-tuning, where dimensionality and resource constraints amplify the difficulty. 

Several recent works explore leveraging zeroth-order optimization (ZO) for LLM quantization~\citep{zhou2025quzo, shang2025fine}. Our approach differs from them in three key aspects. First, our forward pass strictly follows the standard low-bit quantization used in conventional QAT, ensuring that weights remain true low-bit representations. In contrast, \citet{zhou2025quzo} uses different quantizers for perturbation and weight updates, resulting in forward-pass weights that are effectively floating-point values. Similarly, \citet{shang2025fine} adopts different scales for quantization and dequantization, which also breaks the strict low-bit constraint. As a result, these methods cannot fully leverage integer-level multiplications in low-bit computation kernels. Second, with the proposed adaptive outlier smoothing and clipping, jointly optimized with weights during training, our method is the first ZO-based framework supporting simultaneous low-bit quantization of both weights and activations, while prior work only supports weight quantization at moderate bit widths. Third, existing ZO-based quantization studies do not compare against state-of-the-art LLM quantization methods. Consequently, it is difficult to assess how ZO-based approaches perform relative to established quantization techniques. By conducting comprehensive comparisons with representative PTQ and QAT methods, our work provides the first systematic evaluation of ZO-based quantization under standard LLM quantization benchmarks, serving as a reference point for future research that combines ZO optimization with quantization.

\section{Experimental Settings}
\label{settings}

\textbf{Quantization settings.}  To comprehensively evaluate our method, we consider both weight-only and weight-activation quantization, as they represent distinct deployment scenarios. For weight-activation quantization, we adopt per-channel weight quantization and per-token activation quantization, following prior work~\citep{dettmers2022gpt3, shao2023omniquant}. For weight-only quantization, we apply a group-wise strategy, where the weight matrix is partitioned into groups of a fixed size, and each group is assigned its own scale and zero point. Formally, for example, W2A16g128 refers to 2-bit weight-only quantization with 128 as the group size. When $g$ is omitted (e.g., W2A16), the default group size is set to the number of channels, corresponding to per-channel quantization.

\subsection{Hyperparameter Setting}
\label{hyper_setting}

\begin{table*}[htbp]
\small
\centering
\caption{The hyperparameter for experiments. For ZeroQAT, we only show the setting of extra hyperparameters, and have the same setting in other common hyperparameters as MeZO~\citep{malladi2023fine}. 
}
\vspace{5pt}
\begin{tabular}{lcc}
\toprule
Experiment                              & Hyperparameters & Values         \\ \midrule
\multirow{7}{*}{Quantized Pre-training} & Batch size      & 4              \\
                                        & Iteration       & 10K            \\
                                        & Learning rate   & \{5e-7, 1e-8\} \\
                                        & Lr for smothing & 5e-6           \\
                                        & Lr for clipping & 1e-5           \\
                                        & Lr schedule     & Linear Decay   \\
                                        & $\epsilon$ in ZO   & \{1e-3, 5e-4\, 1e-4\} \\ \midrule
\multirow{5}{*}{Quantized Fine-tuning}  & Batch size      & \{32, 16\}     \\
                                        & Iteration       & 8K             \\
                                        & Learning rate   & \{1e-6, 5e-7\} \\
                                        & Lr schedule     & Constant       \\
                                        & $\epsilon$ in ZO   & 1e-3           \\ \bottomrule
\end{tabular}
\label{setting}
\end{table*}

We use the hyperparameters in Table~\ref{setting} for experiments on quantized pre-training and quantized fine-tuning. Specifically, pre-training prefers a smaller learning rate and a smaller perturbation for stable convergence, while for fine-tuning, we can use more aggressive optimization. Moreover, larger models prefers smaller learning rate and a smaller perturbation, while smaller models tend to have the opposite.

\subsection{Settings of Quantized Pre-training}
\label{qpt_setting}

\textbf{Training and evaluation.} Zeroth-order optimization has been shown to benefit from strong initialization~\citep{malladi2023fine}. To provide a stable starting point, we adopt a lightweight initialization strategy based on channel-wise scaling and shifting. Specifically, we pre-train quantized models with OmniQuant~\citep{shao2023omniquant} for a few epochs (2 epochs in the W4A4 setting and 4 epochs in the W2A16 setting), which corresponds to roughly 10\% of the full OmniQuant training cost. This initialization enables ZO to more effectively refine the quantization scales and shift factors. But for the LLama-series weight-only quantization, we remove the smoothing and only maintain clipping, as smoothing only provides limited improvement. For quantized pre-training, we randomly select token segments with length 2048 and then calculate perplexity over WikiText2~\citep{merity2016pointer}, PTB~\citep{marcus1994penn}, and C4~\citep{raffel2020exploring}. To avoid overfitting on one specific dataset, half of the segment samples are from WikiText2 and half from C4, while the total data size is kept the same as in previous work~\citep{shao2023omniquant,dettmers2022gpt3} and set as 128. We further assess zero-shot accuracy on a range of tasks, including PIQA~\citep{bisk2020piqa}, ARC~\citep{clark2018think}, HellaSwag~\citep{zellers2019hellaswag}, Winogrande~\citep{sakaguchi2021winogrande}. We adhere to the GPTQ~\citep{frantar2022gptq} settings for language generation experiments, and leverage the lm-eval-harness~\citep{eval-harness} tool for the evaluation of all zero-shot tasks. 

\textbf{Baselines.} We mainly compared with PTQ methods. For weight-only quantization, we compare with the vanilla round-to-nearest (RTN), GPTQ~\citep{frantar2022gptq}. For weight-activation quantization, we compare our method with SmoothQuant~\citep{xiao2023smoothquant}, RPTQ~\citep{yuan2023rptq}, OutlierSupression+ (OS+)~\citep{wei2022outlier}, OmniQuant~\citep{shao2023omniquant}, and one QAT method LLM-QAT~\citep{liu2023llm}. We keep the quantization setting of SmoothQuant and Outlier Suppression+ with per-channel weight quantization and per-token activation quantization for fair comparisons. Moreover, we compare with two more advanced spin transformation-based quantization methods, Spinquant~\citep{liuspinquant} and Flatquant~\citep{sunflatquant}.

For the quantized pre-training setting, we do not compare with backpropagation-based QAT methods due to feasibility constraints. Specifically, pre-training uses long sequences (e.g., 2048 tokens), and backpropagation-based QAT requires storing full activations, whose memory cost grows quadratically with sequence length. In practice, in the pre-training setting, training a 4-bit 7B model with QAT like QLoRA requires over 80 GB of GPU memory, which is infeasible under our target low-resource setting.  In contrast, in QAT-feasible scenarios (i.e., quantized fine-tuning with shorter sequences), we do compare against multiple QAT baselines, shown in quantized fine-tuning experiments.

\subsection{Settings of Quantized fine-tuning}
\label{qft_setting}

\textbf{Training and evaluation.} Following existing works~\citep{chen2024efficientqat, pan2024lisa}, we fine-tune models on a small subset of the Alpaca dataset~\citep{taori2023stanford}, and report the average accuracy on datasets including MMLU, PIQA, ARC, HellaSwag, and Winogrande. Moreover, we fine-tune and evaluate on two classification datasets, SST-2~\citep{socher2013recursive} and CB~\citep{de2019commitmentbank}, and two question answering datasets, SQuAD~\citep{rajpurkar2016squad} and DROP~\citep{dua2019drop}. For these tasks, we randomly sample 1,000 examples for training, 500 for validation, and 1,000 for testing, following the common few-shot fine-tuning protocol~\citep{malladi2023fine}. Performance is measured using accuracy for classification tasks and F1 scores for question answering tasks. The initialization of quantization parameters is identical to that used in quantized pre-training, and these parameters are frozen during fine-tuning. This design allows us to directly perform quantized fine-tuning without an additional quantized pre-training stage. For all fine-tuning experiments, we run our experiments three times with different seeds and report the average results.

\textbf{Baselines.} Besides the baseline methods used in quantized pre-training (in Section~\ref{Intrinsic}), we additionally compare our method with several leading QAT methods, including QLoRA~\citep{dettmers2023qlora}, QA-LoRA~\citep{xu2023qa}, PEQA~\citep{kim2023memory}, IR-QLoRA~\citep{qin2024accurate}, EfficientQAT~\citep{chen2024efficientqat}, and ZO-based method QuZO~\citep{zhou2025quzo}.

\section{Ablation study on ZeroQAT}
\label{ablation}

In this section, we conduct a comprehensive ablation study on ZeroQAT to illustrate the effectiveness of the components or strategies we used. Specifically, the results include: 

\begin{itemize}
    \item Effect of learnable outlier smoothing and weight clipping (Table~\ref{abl_leanable}).
    \item Effect of using fine-tuned checkpoint by first-order as PTQ's starting point (Table~\ref{PTQ_start}).
    \item Effect of using lightweight ZeroQAT for quantized pre-training (Table~\ref{quantized_ft}).
    \item Effect of the layer selection in lightweight ZeroQAT (Table~\ref{ablation_layer}).
    \item Effect of quantize parameter initialization (Table~\ref{abl_init}) 
    \item Effect of the number of training samples (Table~\ref{abl_samples}).
\end{itemize}

\subsection{Effect of strategies for learnable outlier mitigation}

In ZeroQAT, we introduce a learnable smoothing scaling, transformation, and clipping threshold to effectively relieve the outlier issue in low-bit quantization. We conduct experiments to ablate the effectiveness of these two learnable components. As shown in Table~\ref{abl_leanable}, both components positively influence performance, but learnable smoothing proves essential for weight activation quantization. Disabling it for W4A4 results in a marked increase in perplexity, mainly due to challenges with activation quantization outliers. For weight-only quantization, smoothing only offers a slight improvement for fewer outliers occur~\citep{shao2023omniquant}, therefore, the smoothing is not used for weight-only quantization. 

\begin{table*}[htbp]
\centering
\small
\caption{Effect of each component. WikiText2 perplexity is reported in this table. W/O indicates removing the corresponding learnable components. Here, smoothing indicates scaling and transformation.}
\begin{tabular}{lcccc}
\toprule
\textbf{PPL~$\downarrow$}     & \multicolumn{2}{c}{\textbf{LLama-1-7B}} & \multicolumn{2}{c}{\textbf{LLama-2-13B}} \\ 
\textbf{Leanable Components} & W4A4               & W2A16             & W4A4              & W2A16            \\ \midrule
Smoothing + Clipping           & 6.17     & 9.65     & 5.27    & 8.67   \\ \midrule
W/O Smoothing                & 1.4e3              & 29.61              & 5.2e3             & 15.97            \\
W/O Clipping                 &  16.64             & 9.4e3              & 18.7             & 2.8e3            \\
W/O Smoothing \& Clipping    & 2.1e3              & 1.2e4             & 1.7e4             & 4.6e3            \\ \bottomrule
\end{tabular}
\label{abl_leanable}
\end{table*}

\subsection{Effect of using first-order fine-tuned model for PTQ}

When comparing our method with PTQ methods, the starting point is the full-precision fine-tuned model via ZO, therefore, we investigate if the PTQ method can perform better when using a fine-tuned model by first-order optimization. As shown in Table~\ref{PTQ_start}, when using the first-order fine-tuned model as a starting point, the memory cost of fine-tuning will dramatically increase to around 100 GB, while also not enhancing the performance of PTQ, yielding much lower accuracy compared with FP ZO and ZeroQAT.

\begin{table*}[htbp]
\setlength{\tabcolsep}{2.6pt}
\small
\centering
\caption{Compare with the PTQ method with different fine-tuned models as starting points. Results of fine-tuning OPT-6.7B under W4A4 setting. ZO and FO indicate the starting fine-tuned checkpoint is from first-order and zeroth-order optimization, respectively.}
\begin{tabular}{lcccccc}
\toprule
\textbf{Method} & \textbf{Fine-tuning Memory} & \textbf{PTQ memory} & \textbf{SST-2} & \textbf{CB} & \textbf{SQuAD} & \textbf{DROP} \\ \midrule
FP ZO           & 14.2 GB                     & -                   & 90.2           & 71.4        & 76.0           & 26.4          \\
OmniQuant (ZO)  & 14.2 GB                     & 4.4 GB              & 61.2           & 48.2        & 24.7           & 11.7          \\
OmniQuant (FO)  & 98.6 GB                      & 4.4 GB              & 58.7           & 55.3        & 31.8           & 13.5          \\
ZeroQAT         & 3.7 GB                      & -                   & 88.6           & 67.4        & 61.3           & 22.4          \\ \bottomrule
\end{tabular}
\label{PTQ_start}
\end{table*}

\subsection{Effect of using lightweight ZeroQAT for quantized pre-training}
\label{abl_lightweight}

In ZeroQAT fine-tuning, we devise a lightweight variant that keeps the query and value matrices in full precision while freezing and quantizing the remaining parameters. This design substantially reduces memory cost without sacrificing downstream task accuracy. However, when we apply the same strategy in quantized pre-training, we observe a clear performance drop, as shown in Table~\ref{quantized_ft}. For example, on WikiText2, lightweight ZeroQAT yields perplexity of 41.05 and 21.97 for LLama2-7B and Llama2-13B under W2A16, compared to 29.61 and 15.95 without the lightweight strategy.

This degradation can be attributed to the different optimization dynamics in pre-training versus fine-tuning. Pre-training requires updating a much larger parameter space to capture broad linguistic patterns. Freezing most of the model limits the ability to adapt quantization parameters and compensate for quantization noise, leading to accumulated errors and higher perplexity. In contrast, fine-tuning operates on narrower task-specific distributions, where updating Q and V alone is sufficient to preserve performance. These results highlight that while selective fine-tuning is effective for downstream adaptation, full-parameter optimization remains crucial in the pre-training stage under quantization.

\begin{table*}[htbp]
\small
\centering
\caption{Effect of using lightweight ZeroQAT in quantized pre-training. LW indicates lightweight. Perplexity on Wikitext2 is reported.}
\begin{tabular}{lcccc}
\toprule
\textbf{PPL~$\downarrow$}    & \multicolumn{2}{c}{\textbf{LLama2-7b}} & \multicolumn{2}{c}{\textbf{LLama2-13b}} \\
\textbf{Method} & W2A16              & W4A4              & W2A16               & W4A4              \\ \midrule
ZeroQAT (LW)         &     41.05                 &  19.34                 & 21.97                & 15.45  \\
ZeroQAT        &   29.61                 &  12.95                 & 15.95                & 10.41 \\ \bottomrule                 
\end{tabular}
\label{quantized_ft}
\end{table*}

\subsection{Effect of fine-tuning layer selection.} 
We propose a lightweight variant of ZeroQAT that fine-tunes only the query (Q) and value (V) matrices in the attention layers, while freezing and quantizing the remaining parts of the model to reduce memory overhead. To evaluate the effectiveness of this strategy, we compare it with different layer selection strategies, and the results are reported in Table~\ref{ablation_layer}. The results show that this selective fine-tuning approach achieves a favorable trade-off between performance and memory efficiency: it maintains accuracy comparable to full-parameter fine-tuning, while reducing memory usage to 27\%-38\% of the full-parameter baseline, depending on the model size. This demonstrates that restricting updates to Q and V matrices provides substantial efficiency gains without significant loss of performance.

\begin{table*}[htbp]
\small
    \centering
    \caption{Ablation study for selecting which layers to maintain full-precision and update in Quantized Fine-tuning. The highlighted line with a blue rectangle is the setting used in ZeroQAT. Attn\_Q: attention Query layer; Attn\_V: attention Value layer; Attn\_K: attention Key layer; Attn\_O: attention output projection; Dense: dense fully connected layer.}
    \vspace{5pt}
    \begin{tikzpicture}
        \node (table) [inner sep=0pt] { 
        \setlength{\tabcolsep}{6.0pt}
            \begin{tabular}{ccccccccc}
\toprule
\multicolumn{1}{l}{\multirow{2}{*}{\textbf{Attn\_Q}}} & \multicolumn{1}{l}{\multirow{2}{*}{\textbf{Attn\_V}}} & \multicolumn{1}{l}{\multirow{2}{*}{\textbf{Attn\_K}}} & \multicolumn{1}{l}{\multirow{2}{*}{\textbf{Attn\_O}}} & \multicolumn{1}{l}{\multirow{2}{*}{\textbf{Dense}}} & \multicolumn{2}{l}{\textbf{~~W2A16g128}} & \multicolumn{2}{c}{\textbf{W4A4}} \\
\multicolumn{1}{l}{}                                  & \multicolumn{1}{l}{}                                  & \multicolumn{1}{l}{}                                  & \multicolumn{1}{l}{}                                  & \multicolumn{1}{l}{}                                & Acc.    & \multicolumn{1}{l}{Memory}   & Acc. & \multicolumn{1}{l}{Memory} \\ \midrule
\cmark                                 & \cmark                                 & \cmark                                 & \cmark                                 & \multicolumn{1}{c|}{\cmark}          & 60.4    & 100\%                         & 59.1 & 100\%                       \\
\cmark                                 & \cmark                                 & \cmark                                 & \cmark                                 & \multicolumn{1}{c|}{\xmark}          & 59.3    & 42\%                         & 58.6 & 50\%                       \\
\cmark                                 & \cmark                                 & \cmark                                 & \xmark                                 & \multicolumn{1}{c|}{\xmark}          & 59.4    & 34\%                         & 60.2 & 44\%                       \\
\cmark                                 & \cmark                                 & \xmark                                 & \xmark                                 & \multicolumn{1}{c|}{\xmark}          & 59.9    & 27\%                         & 59.4 & 38\%                       \\
\cmark                                 & \xmark                                 & \xmark                                 & \xmark                                 & \multicolumn{1}{c|}{\xmark}          & 48.8    & 20\%                         & 50.4 & 32\%                       \\ \bottomrule
\end{tabular}
        };

\draw[blue!60!white, very thick] (-6.5, -1.0) rectangle (6.5,-0.6);
    \end{tikzpicture}
    
    \label{ablation_layer}
\end{table*}

\subsection{Effect of initialization for smoothing parameters.} 
We initialize the smoothing parameters by minimizing reconstruction loss before applying ZO. To examine the impact of initialization quality, we conduct an ablation study by varying the number of initialization epochs, as reported in Table~\ref{abl_init}. The results show that initialization has a clear effect on performance. With 0 epochs of initialization, performance drops across different models, while additional epochs (e.g., 20) can further improve accuracy. However, considering both performance gains and computational cost, we adopt two epochs as the default initialization setting. The initialization in our method is lightweight rather than computationally expensive, which takes about 5 minutes in practice, and the initialized quantization parameters can be reused across different fine-tuning tasks for the same pre-trained model. The reliance on initialization stems from the nature of ZO optimization: since ZO estimates gradients via random perturbations, starting far from a reasonable region can lead to unstable or uninformative updates, especially under low-bit quantization where the objective is highly non-smooth. Also, adapting a pre-init smoothing parameter instead of random smoothing is a common strategy in previous work~\citep{shao2023omniquant,sunflatquant}.

\subsection{Effect of training sample size}

Conventional first-order QAT methods are generally data-inefficient, as they rely on large training datasets to provide stable and accurate gradients. To examine whether ZeroQAT exhibits similar behavior, we vary the number of training samples and report the results in Table~\ref{abl_samples}. Compared to the default setting of 128 samples, changing the sample size has only a minor effect on performance, with most perplexity variations remaining within 0.5. This indicates that, unlike conventional methods, ZeroQAT does not heavily rely on large-scale data for convergence. Instead, since its gradients are estimated through noisy zeroth-order approximations, ZeroQAT benefits more from additional optimization iterations rather than larger datasets.

\begin{table*}[htbp]
\centering
\begin{minipage}[t]{\dimexpr0.61\textwidth\relax}
\centering
\small
\caption{Effect of the number of epochs to initialize the smoothing and clipping parameter using reconstruction loss. Perplexity on WikiText2 is reported. $^{*}$ indicates default setting.}
\begin{tabular}{c|ccc}
\toprule
\textbf{Epochs} & \textbf{LLama1-7B} & \textbf{LLama2-7B} & \textbf{OPT-6.7B} \\ \midrule
0~~ & 14.21 & 16.03 & 15.38 \\
1~~ & 7.34 & 6.96 & 11.53 \\
2$^{*}$ & 6.17 & 5.78 & 11.02 \\
10~~ & 6.06 & 5.71 & 11.17 \\
20~~ & 5.98 & 5.84 & 10.88 \\ 
\bottomrule
\end{tabular}
\label{abl_init}
\end{minipage}
\hfill
\begin{minipage}[t]{\dimexpr0.35\textwidth\relax}
\centering
\small
\caption{Effect of using different numbers of training samples (token segments) for training.}
\begin{tabular}{c|cc}
\toprule
\textbf{Samples} & \textbf{W2A16} & \textbf{W4A4} \\ \midrule
32~~ & 19.31 & 6.12 \\
64~~ & 17.88 & 5.96 \\
128$^{*}$ & 17.18 & 5.78 \\
256~~ & 17.03 & 5.72 \\
512~~ & 17.37 & 5.88 \\ 
\bottomrule
\end{tabular}
\label{abl_samples}
\end{minipage}
\end{table*}

\section{More Quantized Pre-training Results}

To illustrate the generalizability of our method, we conduct quantized pre-training on OPT models, and the results are shown in Table~\ref{tab:opt_quant_results}. For W6A6 quantization, similar to other baselines, ZeroQAT also achieves almost lossless results on three datasets. For a more challenging W4A4 setting, ZeroQAT consistently outperforms other baselines for better adaptation. More results on Qwen and Mistral are shown in Table~\ref{qwen_mistral}.

\begin{table}[htbp]
\centering
\small
\caption{Results of perplexity on Wikitext2 on Qwen and Mistral.}
\label{qwen_mistral}
\begin{tabular}{lccc}
\toprule
Method    & Bit   & Qwen2.5-7B & Mistral-7B \\ \midrule
-         & FP16  & 6.67       & 5.54       \\ \midrule
RTN       & W2A16 & 2.4e4      & 4.7e4      \\
OmniQuant & W2A16 & 42.84      & 31.28      \\
\rowcolor[gray]{.92}ZeroQAT   & W2A16 & \textbf{18.37}      & \textbf{14.67}      \\ \midrule
OmniQuant & W4A4  & 15.32      & 13.74      \\
FlatQuant & W4A4  & 7.23       & 5.78       \\
\rowcolor[gray]{.92}ZeroQAT   & W4A4  & \textbf{6.98}       & \textbf{5.66}       \\
\bottomrule
\end{tabular}
\end{table}


\begin{table*}[th]
\centering
\small
\setlength{\tabcolsep}{3.5pt}
\renewcommand{\arraystretch}{1}
\caption{Weight-activation quantization results of OPT models on three datasets: WikiText2 (WIKI), Penn Treebank (PT), and C4. RPTQ* represents a variant that quantizes all activations except the softmax output.}
\begin{tabular}{clccccccccc}
\toprule
\multicolumn{2}{l}{\textbf{OPT / PPL $\downarrow$}}                                           & \multicolumn{3}{c}{\textbf{OPT-2.7B}}            & \multicolumn{3}{c}{\textbf{OPT-6.7B}}            & \multicolumn{3}{c}{\textbf{OPT-13B}}             \\ 
\multicolumn{2}{l}{Task}                                                                      & WIKI           & PT             & C4             & WIKI           & PT             & C4             & WIKI           & PT             & C4             \\ \midrule
\multicolumn{1}{l}{FP16}                   & -                                                & 12.47          & 15.13          & 13.16          & 10.86          & 13.09          & 11.74          & 10.13          & 12.34          & 11.20          \\ \midrule
\multirow{5}{*}{W6A6}                      & SmoothQuant                                      & 12.64          & 15.91          & 13.34          & 11.34          & 13.82          & 12.14          & 10.56          & 12.76          & 11.40          \\
                                           & RPTQ                                             & 13.19          & 16.37          & 14.04          & 11.19          & 13.98          & 12.08          & 11.19          & 13.98          & 12.08          \\
                                           & RPTQ*                                            & 12.71          & 15.53          & 13.33          & 10.96          & 13.24          & 11.86          & 10.96          & 13.24          & 11.86          \\
                                           & OmniQuant                                        & 12.62          & \textbf{15.32} & \textbf{13.29} & 10.96          & \textbf{13.20} & 11.81          & 10.21          & \textbf{12.47} & \textbf{11.17} \\
\rowcolor[gray]{.92}                       & ZeroQAT & \textbf{12.58} & 15.46          & 13.32          & \textbf{10.43} & 13.46          & \textbf{11.44} & \textbf{9.93}  & 12.55          & 11.29          \\ \midrule
\multirow{5}{*}{W4A4}                      & SmoothQuant                                      & 131.47         & 107.10         & 120.57         & 1.8e4          & 1.4e4          & 1.5e4          & 7.4e3          & 6.5e3          & 5.6e3          \\
                                           & RPTQ                                             & 11.45          & 14.71          & 13.12          & 12.00          & 15.17          & 12.85          & 12.74          & 15.76          & 14.71          \\
                                           & RPTQ*                                            & 11.45          & 14.71          & 13.12          & 17.83          & 25.10          & 19.91          & 16.45          & 23.01          & 16.80          \\
                                           & OmniQuant                                        & 15.65          & 23.69          & 16.51          & 12.24          & 15.54          & 13.56          & 11.65          & 15.89          & 13.46          \\
\rowcolor[gray]{.92}                       & ZeroQAT & \textbf{13.11} & \textbf{18.74} & \textbf{14.32} & \textbf{11.02} & \textbf{13.77} & \textbf{12.82} & \textbf{10.53} & \textbf{13.94} & \textbf{12.02} \\ \toprule
\end{tabular}
\label{tab:opt_quant_results}
\end{table*}

\section{Evaluation on MMLU and GSM8K}
\label{mmlu_eval}

To demonstrate the generalizability of ZeroQAT in more realistic and challenging scenarios, we evaluate our method on MMLU, fine-tuning on the Alpaca dataset~\citep{taori2023stanford}, and then evaluate. We conduct experiments based on Llama1-7B, and the results are shown in Table~\ref{mmlu_results}. We also evaluate on GSM8K, and the results are shown in the Table.

\begin{table*}[ht]
\centering
\small
\setlength{\tabcolsep}{4pt}
\renewcommand{\arraystretch}{1}
\caption{Weight-only and weight-activation quantization results of LLama models. This table reports the accuracy of 5 zero-shot tasks.
}
\begin{tabular}{llcccccc}
\toprule
\textbf{LLama / Acc $\uparrow$} & \textbf{Method} & \textbf{PIQA} & \textbf{ARC-e} & \textbf{ARC-c} & \textbf{HellaSwag} & \textbf{Winogrande} & \textbf{Avg.}  \\ \hline
\multirow{6}{*}{LLama-1-7B}     & FP16            & 77.47         & 72.38          & 41.46          & 73.00              & 67.07               & 65.26          \\
                                & SmoothQuant     & 49.80         & 30.40          & 25.80          & 27.40              & 48.00               & 38.41          \\
                                & OmniQuant       & 67.38         & 53.87          & 30.63          & 53.12              & 55.25               & 52.15          \\
                                & Spinquant      & 73.81         & 66.58          & 34.16          & 67.02              & 59.95               & 60.30          \\
                                & FlatQuant       & 75.73         & 68.87          & 36.88          & 69.31              & 65.46               & 63.17          \\
\rowcolor[gray]{.92}            & ZeroQAT         & 76.98         & 69.32          & 36.19          & 70.86              & 64.87               & \textbf{63.64} \\ \midrule
\multirow{5}{*}{LLama-2-7B}     & FP16            & 79.05         & 74.54          & 46.16          & 75.98              & 69.06               & 68.96          \\
                                & SmoothQuant     & 54.10         & 31.40          & 24.80          & 29.10              & 50.00               & 37.88          \\
                                & Spinquant      & 75.30         & 65.32          & 39.42          & 71.45              & 63.46               & 62.99          \\
                                & FlatQuant       & 77.26         & 72.05          & 43.26          & 73.64              & 69.53               & \textbf{67.14}          \\
\rowcolor[gray]{.92}            & ZeroQAT         & 76.86         & 68.67          & 42.39          & 74.43              & 68.35               & 66.14 \\ \midrule
\multirow{5}{*}{LLama-3-8B}     & FP16            & 80.74         & 77.57          & 53.50          & 79.12              & 72.93               & 72.77          \\
                                & SmoothQuant     & 53.10         & 33.50          & 25.10          & 28.80              & 49.60               & 38.02          \\
                                & Spinquant      & 76.66         & 71.38          & 45.73          & 74.07              & 66.38               & 66.84          \\
                                & FlatQuant       & 79.16         & 75.80          & 50.00          & 76.80              & 72.69               & 69.89          \\
\rowcolor[gray]{.92}            & ZeroQAT         & 79.48         & 77.24          & 51.08          & 77.06              & 71.55               & \textbf{71.28} \\ \bottomrule
\end{tabular}

\label{tab:llama_quant_acc_main}
\end{table*}

\begin{table*}[]
\centering
\small
\caption{Results of fine-tuning Llama-1-7B on MMLU benchmarks. 5-shot results are reported.}
\label{mmlu_results}
\begin{tabular}{lccccc}
\toprule
LLama-1-7B (FP: 38.41\%) & GPTQ    & EfficientQAT & SmoothQuant & OmniQuant & ZeroQAT \\ \midrule
W2A16                  & 23.71\% & 24.74\%      & -     & 25.65\%   & 32.55\% \\
W4A4                   & - & -      & 24.55\%     & 26.93\%   & 34.73\% \\
\bottomrule
\end{tabular}
\end{table*}

\begin{table*}[]
\centering
\small
\caption{Results of fine-tuning Llama-2-7B on GSM8K benchmarks. 5-shot results are reported.}
\label{gsm8k_results}
\begin{tabular}{lccccc}
\toprule
LLama-2-7B (FP: 27.3\%) & SmoothQuant    & OmniQuant & FlatQuant  & ZeroQAT \\ \midrule
W2A16                  & -\% & 6.3\%           & 18.3\%   & 20.7\% \\
W4A4                   & 0.6\% & 7.8\%      & 18.8\%   & 22.1\% \\
\bottomrule
\end{tabular}
\end{table*}

\section{Implementation and Optimization on Mobile Device}
\label{IO}

To demonstrate the practicality of ZeroQAT in real-world deployment scenarios, we implement and evaluate our framework on a OnePlus 12 smartphone (Snapdragon 8 Gen 3) with a mobile GPU backend. Deploying LLMs on mobile devices is challenging due to strict memory and compute constraints. In particular, our device enforces a 2GB single memory pool limitation, which prevents many existing inference frameworks (e.g., llama.cpp) from directly running large-model low-bit inference such as OPT-1.3B. To address this constraint, we redesign the memory management strategy by splitting model weights into multiple smaller memory pools and carefully managing runtime buffers.

In addition, we adopt several practical engineering optimizations to improve execution efficiency on mobile hardware. These include operator fusion and execution optimizations following prior work DNNFusion~\citep{niu2021dnnfusion}, low-bit weight packing, and data layout optimizations inspired by SmartMem~\citep{niu2024smartmem} to reduce memory overhead and improve memory locality. Since our workflow performs inference-only execution without backpropagation, intermediate activations can be released after each layer finishes computation, further reducing peak memory usage. Therefore, achieving significantly low memory usage in on-device fine-tuning, as shown in Table~\ref{tab:eva-on-device}.







~
\clearpage

\onecolumn
\section{Theoretical Analysis}
\label{theory}

\begin{proposition}[Unbiasedness and explicit second-moment bound for the two-point ZO estimator]
\label{prop:unbiased_variance}
Let $Q:\mathbb{R}^d\!\to\!\mathbb{Z}^d$ be the per-coordinate uniform quantizer of step size $\Delta>0$
(rounding with optional clipping/zero-point), and let $L(\cdot;B)$ be $G$-Lipschitz in its argument
with respect to $\ell_2$: $|L(z;B)-L(z';B)|\le G\|z-z'\|_2$ for all $z,z'$ and all mini-batches $B$.
For $\varepsilon>0$ define the Gaussian-smoothed (forward-only) objective
\[
f_\varepsilon(W)\;=\; \mathbb{E}_{u\sim\mathcal{N}(0,I_d)}\,\mathbb{E}_B\,
L\!\big(Q(W+\varepsilon u);B\big)\,,
\]
and the two-point ZO estimator with $q$ i.i.d.\ directions $u_i\sim\mathcal{N}(0,I_d)$:
\[
g_b(W;B)\;=\;\frac{1}{q}\sum_{i=1}^q
\frac{L\!\big(Q(W+\varepsilon u_i);B\big)-L\!\big(Q(W-\varepsilon u_i);B\big)}{2\varepsilon}\;u_i\,.
\]
Assume $\mathbb{E}_B\!\left[\,\left|L(Q(W+\varepsilon u);B)\right|\,\right]<\infty$ for all $W$ and $\varepsilon>0$.
Then:
\begin{enumerate}
\item[\textnormal{(i)}] \textbf{Unbiasedness.} The estimator targets the gradient of the smoothed objective:
\[
\mathbb{E}_{u,B}\big[g_b(W;B)\big]\;=\;\nabla f_\varepsilon(W)\,.
\]
\item[\textnormal{(ii)}] \textbf{Mean-squared error bound.} Writing the expectation over all randomness $(u,B)$,
\[
\mathbb{E}\,\big\|\,g_b(W;B)-\nabla f_\varepsilon(W)\,\big\|_2^2
\;\le\;\frac{1}{q}\!\left[
2G^2\,d(d{+}2)\;+\;\frac{G^2\,\Delta^2\,d^2}{2\,\varepsilon^2}
\right].
\]
In particular, ignoring the quantizer offset term (formally $\Delta{=}0$), the estimator’s MSE scales as
$O\!\big(G^2 d^2/q\big)$ under standard Gaussian directions.
\end{enumerate}
\end{proposition}

\begin{proof}
\emph{(i) Unbiasedness.}
Let $U\sim\mathcal{N}(0,I_d)$ and write $Z=W+\varepsilon U$. Then
$f_\varepsilon(W)=\mathbb{E}_{Z,B}\,L(Q(Z);B)$ with $Z\sim\mathcal{N}(W,\varepsilon^2 I_d)$.
Differentiating under the integral with respect to the mean of the Gaussian and using
$\nabla_W \log p_{W,\varepsilon}(Z)=(Z-W)/\varepsilon^2$,
\[
\nabla f_\varepsilon(W)\;=\;\mathbb{E}_{Z,B}\!\left[\frac{Z-W}{\varepsilon^2}\,L(Q(Z);B)\right]
\;=\;\frac{1}{\varepsilon}\,\mathbb{E}_{U,B}\!\left[U\,L\!\big(Q(W+\varepsilon U);B\big)\right].
\]
By antithetic symmetry of $U$,
\[
\mathbb{E}_{U,B}\!\left[\frac{L(Q(W+\varepsilon U);B)-L(Q(W-\varepsilon U);B)}{2\varepsilon}\,U\right]
\;=\;\frac{1}{\varepsilon}\,\mathbb{E}_{U,B}\!\left[U\,L\!\big(Q(W+\varepsilon U);B\big)\right],
\]
hence $\mathbb{E}_{u,B}[g_b(W;B)]=\nabla f_\varepsilon(W)$.

\smallskip
\emph{(ii) Second-moment/MSE bound.}
Let
\[
g(W;B,U)\;:=\;\frac{L(Q(W+\varepsilon U);B)-L(Q(W-\varepsilon U);B)}{2\varepsilon}\;U\,.
\]
Using independence of the $q$ i.i.d.\ samples,
\(
\mathbb{E}\,\|g_b-\nabla f_\varepsilon\|_2^2
\le \frac{1}{q}\,\mathbb{E}\,\|g-\mathbb{E}g\|_2^2
\le \frac{1}{q}\,\mathbb{E}\,\|g\|_2^2.
\)
By $G$-Lipschitzness of $L(\cdot;B)$ and triangle inequality for $Q$,
\[
\|g\|_2
\;\le\;\frac{G}{2\varepsilon}\,\big\|Q(W+\varepsilon U)-Q(W-\varepsilon U)\big\|_2\,\|U\|_2
\;\le\;G\|U\|_2^2+\frac{G\,\Delta\sqrt{d}}{2\varepsilon}\,\|U\|_2,
\]
where we used the standard quantization geometry
$\|Q(x)-Q(y)\|_2\le\|x-y\|_2+\|Q(x)-x\|_2+\|Q(y)-y\|_2
\le\|x-y\|_2+\Delta\sqrt{d}$ and $\|Q(z)-z\|_2\le(\Delta/2)\sqrt{d}$.
Applying $(a{+}b)^2\le 2a^2+2b^2$ and Gaussian moment identities
$\mathbb{E}\|U\|_2^2=d$, $\mathbb{E}\|U\|_2^4=d^2+2d$ yields
\[
\mathbb{E}\,\|g\|_2^2
\;\le\;2G^2\,\mathbb{E}\|U\|_2^4\;+\;\frac{G^2\,\Delta^2 d}{2\varepsilon^2}\,\mathbb{E}\|U\|_2^2
\;=\;2G^2\,\big(d^2+2d\big)\;+\;\frac{G^2\,\Delta^2\,d^2}{2\,\varepsilon^2}.
\]
Dividing by $q$ completes the proof.
\end{proof}



\paragraph{What STE assumes and why it is biased.}
The straight-through estimator (STE) replaces the ill-defined Jacobian $J_Q(W)$ of the
piecewise-constant quantizer $Q$ by a hand-crafted surrogate $S(W)$ (e.g., $S(W)=I$ or a clipped
indicator). The chain rule then yields the surrogate update.
\[
g_{\mathrm{STE}}(W;B) \;=\; S(W)^\top \,\nabla_Q L\!\big(Q(W);B\big).
\]
Because $Q$ is flat almost everywhere, the true chain rule gives $J_Q(W)=0$ a.e., so asserting
$J_Q(W)\approx S(W)$ implicitly enforces \emph{gradient invariance to the discrete parameterization}:
$\nabla_W L(Q(W);B)\approx \nabla_Q L(Q(W);B)$ regardless of whether small perturbations of $W$
actually change $Q(W)$. This mismatch makes $g_{\mathrm{STE}}$ a biased estimator of any well-defined target (e.g., $\nabla f_\varepsilon(W)$ from Gaussian smoothing, or Clarke’s generalized gradient of
$f$), and the bias can remain large away from quantization thresholds, where the true smoothed
gradient vanishes in magnitude. 

\begin{proposition}[Worst-case STE bias in expectation, 1-D]
\label{prop:ste-bias}
Assume $d=1$ and a uniform $b$-bit quantizer of step $\Delta>0$. Let $L(z;B)=G\,z$ be a
$G$-Lipschitz linear loss in its (quantized) argument. For $W\in\mathbb R$, let $r(W)$ be the distance
to the nearest quantization threshold and set $t := r(W)/\varepsilon$. Consider the common STE
choice $S(W)\equiv 1$. Then, for every $W$ and $\varepsilon>0$,
\[
\Big\|\,\mathbb E_B\big[g_{\mathrm{STE}}(W;B)\big] \;-\; \nabla f_\varepsilon(W)\,\Big\|
\;\;\ge\;\;
G \;-\; \frac{G}{\sqrt{2\pi}}\Big(\tfrac{\Delta}{\varepsilon}+2t+\tfrac{2}{t}\Big)\,e^{-t^2/2}.
\]
In particular, for any $\delta\in(0,1)$, if
\[
t \;\ge\;
\sqrt{\,2\log\!\left(\frac{1}{\sqrt{2\pi}\,\delta}\Big(\tfrac{\Delta}{\varepsilon}+2t+\tfrac{2}{t}\Big)\right)},
\]
then $\big\|\,\mathbb E_B[g_{\mathrm{STE}}]-\nabla f_\varepsilon(W)\,\big\|\ge (1-\delta)G$; i.e.,
the STE exhibits an $\Omega(G)$ bias in expectation away from thresholds. 
\end{proposition}

\begin{proof}[Proof via two lemmas]
We first record two standard ingredients.

\begin{lemma}[1-D Gaussian tail identities]
\label{lem:gaussian-tails}
If $U\sim\mathcal N(0,1)$ and $t\ge 0$, then

    \begin{equation}
        \begin{aligned}
            & \mathbb E\!\big[|U|\mathbf 1\{|U|\ge t\}\big]=2\phi(t),\qquad \\
            & \mathbb E\!\big[U^2\mathbf 1\{|U|\ge t\}\big]=2\big(t\phi(t)+1-\Phi(t)\big),\qquad\\
            & \mathbb P(|U|\ge t)=2(1-\Phi(t)),
        \end{aligned}
    \end{equation}

and Mills’ bound $1-\Phi(t)\le \phi(t)/t$ holds for $t>0$, where
$\phi(t)=\frac{1}{\sqrt{2\pi}}e^{-t^2/2}$ and $\Phi$ is the standard normal cdf.  \qedhere
\end{lemma}

\begin{lemma}
\label{lem:grad-decay}
Let $d=1$ and $t=r(W)/\varepsilon$. If $L(\cdot;B)$ is $G$-Lipschitz in its argument, then
\[
\big\|\nabla f_\varepsilon(W)\big\|
\;\le\;
\frac{G}{\sqrt{2\pi}}
\Big(\tfrac{\Delta}{\varepsilon}+2t+\tfrac{2}{t}\Big)\,e^{-t^2/2}.
\]
\emph{Proof.}
From the Gaussian-smoothing representation (two-point form),
\[
\nabla f_\varepsilon(W)
=\mathbb E_{u,B}\!\left[\frac{L(Q(W+\varepsilon u);B)-L(Q(W-\varepsilon u);B)}{2\varepsilon}\,u\right].
\]
By $G$-Lipschitzness and symmetry,
\[
\|\nabla f_\varepsilon(W)\|
\le \frac{G}{2\varepsilon}\,\mathbb E_u\!\left[\,|u|\,\big|Q(W+\varepsilon u)-Q(W-\varepsilon u)\big|\,\right].
\]
If $|u|<t$, both perturbations stay in the same quantization cell and the difference vanishes; otherwise,
the quantization geometry yields $\big|Q(W+\varepsilon u)-Q(W-\varepsilon u)\big|
\le (2\varepsilon|u|+\Delta)\mathbf 1\{|u|\ge t\}$. Hence
\[
\|\nabla f_\varepsilon(W)\|
\le \frac{G}{2\varepsilon}\,\mathbb E\!\left[(2\varepsilon|u|+\Delta)|u|\mathbf 1\{|u|\ge t\}\right].
\]
Expanding and applying Lemma~\ref{lem:gaussian-tails} (with Mills’ bound) gives the claim. \qed
\end{lemma}

We now prove the proposition. For the stated STE with $S(W)\equiv 1$ and the linear loss
$L(z;B)=Gz$, one has $\nabla_Q L(Q(W);B)\equiv G$, hence
\[
\mathbb E_B\!\big[g_{\mathrm{STE}}(W;B)\big] \;=\; G.
\]
Therefore,
\[
\Big\|\,\mathbb E_B[g_{\mathrm{STE}}]-\nabla f_\varepsilon(W)\,\Big\|
\;\ge\;
G-\big\|\nabla f_\varepsilon(W)\big\|
\;\stackrel{\text{(Lemma \ref{lem:grad-decay})}}{\ge}\;
G-\frac{G}{\sqrt{2\pi}}\Big(\tfrac{\Delta}{\varepsilon}+2t+\tfrac{2}{t}\Big)\,e^{-t^2/2}.
\]
Rearranging yields the thresholded $(1-\delta)G$ lower bound.  \qedhere
\end{proof}

\begin{remark}
This formalizes the violation of gradient invariance
and explains the $\Omega(G)$ expected bias away from thresholds. Multidimensional extensions follow
by coordinate-wise threshold distances and union/tail bounds. 
\end{remark}

\end{document}